\documentclass{article}
\usepackage{arxiv}

\usepackage{cite}

\usepackage{amsmath, amssymb, amsfonts}
\usepackage{mathtools}
\usepackage{bm}

\usepackage{amsthm}
\theoremstyle{plain}

\newtheorem{lemma}{Lemma}
\theoremstyle{remark}

\usepackage{graphicx}
\usepackage{subcaption}
\usepackage{booktabs}

\usepackage{xcolor}

\usepackage[hidelinks]{hyperref}
\usepackage[nameinlink]{cleveref}
\crefname{figure}{Fig.}{Figs.}
\Crefname{figure}{Fig.}{Figs.}
\crefname{section}{Section}{Sections}
\Crefname{section}{Section}{Sections}
\crefname{table}{Table}{Tables}
\Crefname{table}{Table}{Tables}
\crefname{lemma}{Lemma}{Lemmas}
\Crefname{lemma}{Lemma}{Lemmas}

\usepackage{etoolbox}
\AtBeginEnvironment{thebibliography}{\renewcommand{\url}[1]{}}

\newcommand{\R}[1]{\mathbb{R}^{#1}}
\newcommand{\M}{\mathcal{M}}
\newcommand{\N}{\mathcal{N}}
\renewcommand{\L}{\mathcal{L}}

\newcommand{\G}{\mathbf{G}}
\newcommand{\Hg}{\mathbf{H}}
\newcommand{\SO}[1]{\mathbf{SO}(#1)}
\newcommand{\SE}[1]{\mathbf{SE}(#1)}
\newcommand{\SEk}[2]{\mathbf{SE}_{#1}(#2)}

\newcommand{\g}{\mathfrak{g}}
\newcommand{\ph}{\mathfrak{h}}
\newcommand{\se}[2]{\mathfrak{se}_{#1}({#2})}

\newcommand{\Ad}{\operatorname{Ad}}
\newcommand{\ad}{\operatorname{ad}}
\newcommand{\T}{\mathsf{T}}

\newcommand{\wed}[1]{{#1}^{\wedge}}
\newcommand{\ve}[2]{{#1}^{\vee}_{#2}}
\newcommand{\cro}[1]{{#1}^{\times}}
\renewcommand{\v}[1]{{#1}^{\vee}}

\newcommand{\zero}[2]{\mathbf{0}_{#1\times#2}}

\hypersetup{
    pdftitle={Equivariant Filter Design for Acoustic and Depth Aided Inertial Navigation Systems},
    pdfauthor={Arihant Lunawat, Pieter van Goor, Frank Dellaert, Stefan B. Williams},
    pdfsubject={Equivariant filtering for autonomous underwater vehicle navigation},
    pdfkeywords={Marine robotics, Inertial navigation system, Symmetry, Equivariance, Equivariant filter},
}

\newcommand{\orcid}[1]{\href{https://orcid.org/#1}{\includegraphics[scale=0.06]{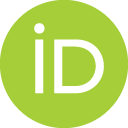}\hspace{1mm}}}

\title{Equivariant Filter Design for Acoustic and Depth Aided Inertial Navigation Systems}
\headertitle{Equivariant Filter Design for Acoustic and Depth Aided Inertial Navigation Systems}

\author{
    \orcid{0009-0001-5908-5403}Arihant Lunawat \\
    The University of Sydney \\
    Sydney, Australia \\
    \texttt{alun0752@uni.sydney.edu.au} \\
    \And
    \orcid{0000-0003-4391-7014}Pieter van Goor \\
    The University of Sydney \\
    Sydney, Australia \\
    \texttt{pieter.vangoor@sydney.edu.au} \\
    \AND
    \orcid{0000-0002-5532-3566}Frank Dellaert \\
    Georgia Tech \\
    Atlanta, USA \\
    \texttt{frank.dellaert@cc.gatech.edu} \\
    \And
    \orcid{0000-0001-9416-5639}Stefan B. Williams \\
    The University of Sydney \\
    Sydney, Australia \\
    \texttt{stefan.williams@sydney.edu.au} \\
}

\newcommand{\publicationversion}{Preprint}
\newcommand{\publicationdetails}{Under review}

\begin{document}

\maketitle

\begin{abstract}
Autonomous Underwater Vehicles (AUVs) navigating without GPS typically fuse inertial measurements with acoustic Doppler Velocity Log (DVL) velocities and pressure-derived depth.
Posing the navigation state on a Lie group improves accuracy and consistency.
However, state-of-the-art filters based on the Invariant Extended Kalman Filter (IEKF) append the Inertial Measurement Unit (IMU) biases as a Euclidean extension, which breaks the group-affine structure required for exact log-linear error dynamics, causing the reported covariance to degrade alongside the estimate.
We apply the Tangent-Group (TG) symmetry, which carries the biases within the geometry of the state space, to derive an Equivariant Filter (EqF) for this system, leaving zero linearization error in the navigation states and second-order error only in the biases.
We develop an equivariant output model for the DVL, whose update incurs only third-order linearization error, together with a direct pressure output.
Monte Carlo simulations benchmark the TG-EqF against a Two-Frame-Group IEKF and a Multiplicative EKF.
The TG-EqF reduces error by 18--25\% against both alternatives in each of attitude, velocity, and position.
The main benefit is in the covariance it estimates: its Average Normalized Estimation Error Squared (ANEES) stays closer to its nominal value of one than that of the others.
Offline analysis on AUV field data corroborates the findings of the simulations, demonstrating reduced position drift.
\end{abstract}


\section{Introduction}
\label{sec:Introduction}

Autonomous Underwater Vehicles (AUVs) operate in GPS-denied settings using dead reckoning, making the navigation filter's covariance as critical as the state estimate to maintain filter consistency during long dives.
Dead reckoning fuses inertial measurements, corrupted by bias and noise, with acoustic and pressure sensing~\cite{potokarInvariantExtendedKalman2021, talInertialNavigationSystem2017}.
The Extended Kalman Filter (EKF) is the standard for Inertial Navigation Systems (INS) on resource-constrained underwater platforms~\cite{talInertialNavigationSystem2017, huTightlyCoupledSINS2025}.
It linearizes the INS model about the estimated trajectory, making both the estimated mean and covariance dependent on the validity of the linearization.

\begin{figure}[!t]
    \centering
    \begin{subfigure}[b]{0.49\linewidth}
        \centering
        \includegraphics[width=\linewidth]{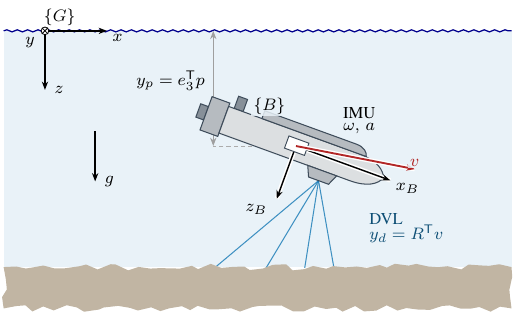}
        \caption{}
        \label{fig:system-overview}
    \end{subfigure}
    \hfill
    \begin{subfigure}[b]{0.49\linewidth}
        \centering
        \includegraphics[width=\linewidth]{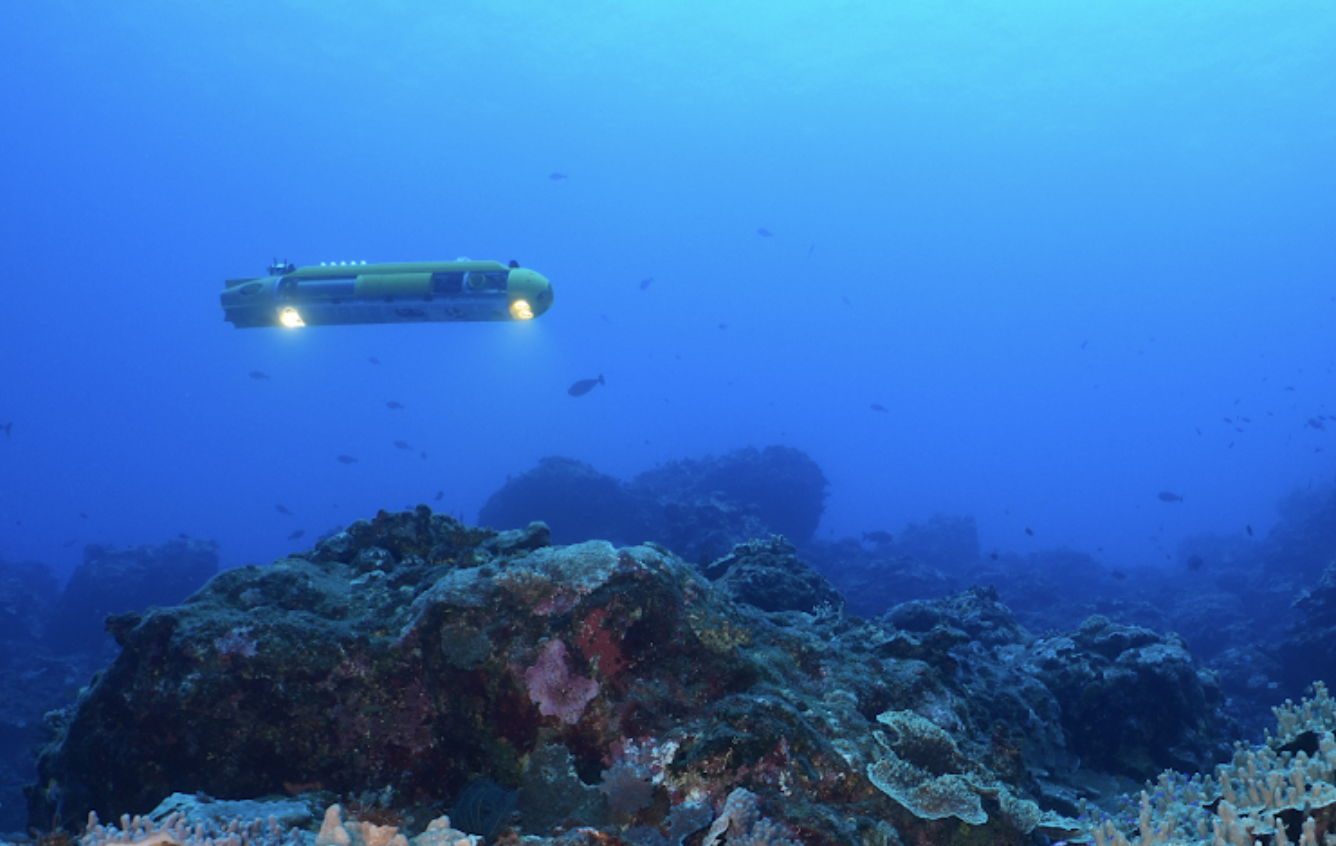}
        \caption{}
        \label{fig:SeekerAUV}
    \end{subfigure}
    \caption{(a) Acoustic and depth-aided inertial navigation. The IMU measures angular velocity $\omega$ and specific force $a$ in the body-fixed frame $\{B\}$, which is not aligned with the gravity-aligned global frame $\{G\}$.
    The DVL reports the vehicle velocity resolved in the body frame, $y_{d} = R^{\T}v$~\eqref{eq:DVLOutput}, and the pressure sensor reports the single gravity-aligned channel $y_{p} = e_{3}^{\T}p$ of the global position~\eqref{eq:DepthOutput}.
    The DVL output admits an equivariant output action; the depth output does not, and is applied directly.
    (b) The AUV platform deployed in the field navigating a reef environment.}
    \label{fig:Overview}
\end{figure}

Posing the INS problem on a Lie group rather than in Euclidean coordinates can improve accuracy and consistency, particularly when the Lie group matches intrinsic symmetries of the system~\cite{barrauInvariantExtendedKalman2017, mahonyObserverDesignNonlinear2022, fornasierEquivariantFilterDesign2022a, barrauGeometryNavigationProblems2023}. 
The Multiplicative EKF (MEKF) exploits the simplest such symmetry, carrying the attitude error on the Lie group $\SO{3}$~\cite{bonin-fontMultisensorAidedInertial2013}, but does not model any coupling between attitude, velocity, and position. 
Barrau and Bonnabel~\cite{barrauInvariantExtendedKalman2017} proposed the Invariant EKF (IEKF) for kinematic systems posed on a matrix Lie group.
They showed that INS, without considering IMU biases, is a \emph{group-affine system} when studied using the Lie group $\SEk{2}{3}$, resulting in a linear evolution of the filter's error coordinates.
Extending the IEKF to a biased IMU requires augmenting the state space with bias states as a Euclidean extension~\cite{potokarInvariantExtendedKalman2021, huTightlyCoupledSINS2025, talInertialNavigationSystem2017}.
This breaks the group-affine property~\cite{fornasierEquivariantSymmetriesInertial2025} of the kinematic system, and the resulting error dynamics are accurate only up to second order~\cite{barrauInvariantExtendedKalman2017}.
The recent Two-Frame-Group (TFG) symmetry~\cite{barrauGeometryNavigationProblems2023} provides an alternative way to include the biases that recovers the group-affine property for straight-line trajectories when the gyroscope bias is known.
However, in general the TFG linearization of velocity and position error dynamics remains strictly second-order accurate~\cite{fornasierEquivariantFilterDesign2022a}.
Wang et al.~\cite{wangInvariantFilteringMethod2024} and Zhu et al.~\cite{zhuTwoframeGroupError2025} applied the TFG symmetry to DVL-aided INS with real-world data, the former also using depth observations.

The Equivariant Filter (EqF)~\cite{vangoorEquivariantFilterEqF2023} relaxes the requirement that the state space be a matrix Lie group with group-affine dynamics and instead provides an observer design for systems that evolve on a homogeneous space, that is, a smooth manifold equipped with a transitive Lie group action.
The EqF has been successfully applied to visual-inertial odometry~\cite{vangoorEqVIOEquivariantFilter2023,fornasierMSCEqFMultiState2024}, GNSS-aided INS~\cite{luIteratedEquivariantFilter2025}, and range-only SLAM~\cite{geEquivariantFilterDesign2025}.
Recently, an EqF based on a new \emph{Tangent-Group} (TG) symmetry was introduced to include IMU biases in a semi-direct product structure~\cite{fornasierEquivariantFilterDesign2022a}.
In contrast to the MEKF, IEKF, and TFG-IEKF, the TG-EqF error dynamics are exactly linear in the navigation states, with second-order terms appearing only in the error dynamics of the bias states~\cite{fornasierEquivariantSymmetriesInertial2025}.
This was shown to significantly improve accuracy and consistency in position-aided INS experiments for aerial systems~\cite{fornasierEquivariantSymmetriesInertial2025}.
The TG-EqF framework is therefore a promising solution for the acoustic- and pressure-aided INS problem encountered in AUV navigation.

In this paper, we build on the TG-EqF~\cite{fornasierEquivariantSymmetriesInertial2025} to design an EqF for an acoustic- and pressure-aided, biased INS for underwater navigation, as illustrated in~\cref{fig:system-overview}.
Our experiments compare the TG symmetry directly against the TFG symmetry used by Wang et al.~\cite{wangInvariantFilteringMethod2024} for the same sensors.
Our main contributions are as follows:
\begin{enumerate}
    \item An EqF for acoustic- and pressure-aided inertial navigation with biased IMU measurements, leveraging the TG symmetry.
    This is the first EqF implementation to fuse biased IMU, DVL, and pressure-sensor data.
    \item An equivariant DVL velocity measurement model, whose update incurs only third-order approximation error.
    \item A simulation study demonstrating the improved consistency and accuracy of the TG-EqF compared to an MEKF and a TFG-IEKF, complemented by an empirical validation on real-world AUV sensor logs.
    \item An open-source C++ implementation of the TG-EqF, TFG-IEKF, and MEKF, based on the GTSAM~\cite{BorglabGtsam2026a} library.\footnote{\url{https://github.com/arihantb2/gtsam/tree/tg-eqf}}
\end{enumerate}
The paper proceeds as follows. 
After reviewing EqF theory (\cref{sec:Preliminaries}), we present the system model and observer (\cref{sec:SystemModel,sec:SymmetryGroup}), filter design (\cref{sec:MeasurementModels,sec:FilterDesign}), and experiments (\cref{sec:Experiments}).
The symmetry, group actions, and equivariant lift are from~\cite{fornasierEquivariantFilterDesign2022a, fornasierEquivariantSymmetriesInertial2025}; our contributions arise from their instantiation for the acoustic suite, the equivariant DVL output model, and an evaluation using both simulated and field data.

\section{Preliminaries}
\label{sec:Preliminaries}

\subsection{Notation}
\label{sec:Notation}

Lie groups are set in bold ($\G$), Lie algebras in Fraktur ($\g$), and smooth manifolds in calligraphic type ($\M$, $\L$, $\N$); their elements are plain italic, uppercase for matrix representations ($T$, $A$, $R$, $W$, $B$, $G$) and lowercase for scalars and coordinate vectors ($v$, $p$, $\omega$, $\eta$, $\zeta$).
The bold symbols are system matrix $\mathbf{A}_{t}^{o}$, output matrix $\mathbf{C}$, covariance $\mathbf{\Sigma}$, process noise matrix $\mathbf{Q}$ and measurement noise matrix $\mathbf{R}$.
A superscript $o$ marks a quantity at the filter origin ($\xi^{o}$, $u^{o}$, $y^{o}$), and $\star$ an equivariant output linearization ($\mathbf{C}^{\star}$).
Dimensions are given once, where an object is introduced, and dropped thereafter.

The cross map $\cro{(\cdot)} : \R{3} \rightarrow \mathfrak{so}(3)$ returns the $3\times3$ skew-symmetric matrix of an $\R{3}$ vector. Transpose is $(\cdot)^{\T}$, $e_{i}$ is the $i$-th standard basis vector of $\R{3}$, and $\mathbb{S}_{+}(n)$ denotes the set of symmetric positive-definite $n \times n$ matrices.
Given a differentiable map $h$, the differentiable is denoted $\mathrm{D} h$.

\subsection{Lie Groups and Semi-Direct Products}
\label{sec:LieGroups}

We follow the conventions for matrix Lie groups laid out in~\cite{fornasierEquivariantSymmetriesInertial2025}.
For a matrix Lie group $\G$ with identity $I_{\G}$, the Lie algebra $\g$ is defined as the tangent space at the identity element and is isomorphic to $\R{n}$, $n = \dim\g$, through the wedge and vee maps.
The wedge map $(\cdot)^{\wedge} : \R{n} \rightarrow \mathfrak{g}$ and its inverse, the vee map $(\cdot)^{\vee} : \mathfrak{g} \to \R{n}$, relate coordinate vectors in $\R{n}$ to the matrix Lie algebra $\g$.
The exponential $\exp : \g \rightarrow \G$ and its local inverse $\log$ relate the Lie group $\G$ and its Lie algebra $\g$.
The Adjoint $\Ad_{X}[\wed{u}] \coloneq X\wed{u}X^{-1}$ and the algebra adjoint $\ad_{\wed{u}}[\wed{v}] \coloneq [\wed{u}, \wed{v}]$ have coordinate matrices $\ve{\Ad}{X}$ and $\ve{\ad}{u}$.

The extended pose group $\SEk{2}{3}$ represents the navigation states $T \coloneq (R, v, p)$ as a $5\times5$ matrix~\cite{fornasierEquivariantSymmetriesInertial2025}. The wedge $\wed{(\cdot)}$ maps vectors in $\R{9}$ to the Lie algebra $\se{2}{3}$,
\begin{align}
\label{eq:SE23}
    \wed{x} = \begin{bmatrix}
        \cro{x_{1}} & x_{2}       & x_{3} \\
        \zero{2}{3} & \zero{2}{1} & \zero{2}{1}
    \end{bmatrix} \in \se{2}{3},
\end{align}
for $x \coloneq (x_{1}, x_{2}, x_{3}) \in \R{9}$, with the inverse vee map $\v{(\cdot)}$.

Given a Lie group $\Hg$ with Lie algebra $\ph$, the semi-direct product (tangent) group $\mathbf{TH} \coloneq \Hg \ltimes \ph$ has elements $X \coloneq (A, c)$, $A \in \Hg$, $c \in \ph$, with product and inverse given by
\begin{subequations}
\label{eq:GroupAxioms}
\begin{align}
    X \cdot Y   &\coloneq \big(AB,\ c + \Ad_{A}[d]\big), & Y &\coloneq (B, d), \\
    X^{-1}      &\coloneq \big(A^{-1},\ -\Ad_{A^{-1}}[c]\big), & I_{\mathbf{TH}} &\coloneq (I_{\mathbf{H}}, 0),
\end{align}
\end{subequations}
so the $\Hg$-factor acts on the $\ph$-factor through its Adjoint. 
For $\gamma = (\eta, \zeta) \in \ph \ltimes \ph$ and $X = (A, c) \in \Hg \ltimes \ph$,
\begin{subequations}
\label{eq:ExpMap}
\begin{align}
    \exp_{\mathbf{TH}}(\wed{\gamma}) &= \Big(\exp_{\Hg}(\wed{\eta}),\ \wed{\big(J_{l}(\eta)\,\zeta\big)}\Big), \\
    \log_{\mathbf{TH}}(X)            &= \Big(\log_{\Hg}(A),\ \wed{\big(J_{l}(\eta_{A})^{-1}\,\v{c}\big)}\Big),
\end{align}
\end{subequations}
where $\eta_{A} \coloneq \v{\log_{\Hg}(A)}$ and $J_{l}(\eta) \coloneq \sum_{k \geq 0} (\ve{\ad}{\eta})^{k}/(k+1)!$ is the left Jacobian~\cite{fornasierEquivariantSymmetriesInertial2025}. 
The group and algebra adjoints inherit the block structure of the semi-direct product,
\begin{align}
\label{eq:Adjoints}
    \ve{\Ad}{X} = \begin{bmatrix}
        \ve{\Ad}{A}                  & \mathbf{0} \\
        \ve{\ad}{\v{c}}\,\ve{\Ad}{A} & \ve{\Ad}{A}
    \end{bmatrix}, &&
    \ve{\ad}{\gamma} = \begin{bmatrix}
        \ve{\ad}{\eta}  & \mathbf{0} \\
        \ve{\ad}{\zeta} & \ve{\ad}{\eta}
    \end{bmatrix}.
\end{align}

\subsection{Group Actions, Homogeneous Spaces and Equivariance}
\label{sec:GroupActions}

A right group action $\phi$ of a Lie group $\G$ on a smooth manifold $\M$ is a smooth map $\phi : \G \times \M \rightarrow \M$ satisfying
\begin{align}
\label{eq:RightActionAxioms}
    \phi(I_{\G}, \xi) = \xi, \qquad \phi\big(Y, \phi(X, \xi)\big) = \phi(XY, \xi),
\end{align}
$\forall X, Y \in \G,\xi \in \M$. 
It induces a family of diffeomorphisms $\phi_{X} : \M \rightarrow \M$ and projections $\phi_{\xi} : \G \rightarrow \M$. 
The action is \emph{transitive} if every projection $\phi_{\xi}$ is surjective, in which case $\M$ is a \emph{homogeneous space} of $\G$; it is \emph{free} if $\phi(X, \xi) = \xi$ implies $X = I_{\G}$.
A system $\dot{\xi} = f_{u}(\xi)$ with input $u$ is \emph{equivariant} under the action $\phi$ and a compatible input action $\psi : \G \times \L \rightarrow \L$ when
\begin{align}
\label{eq:EquivarianceDefinition}
    f_{\psi(X, u)}\big(\phi(X, \xi)\big) = \mathrm{D}\phi_{X} \big[f_{u}(\xi)\big]
\end{align}
$ \forall X \in \G,\xi \in \M,u \in \L$.

\subsection{Equivariant Filter Framework}
\label{sec:EqFBackground}

For kinematic systems on a homogeneous space, the EqF poses the observer state on the symmetry group, linearizing global error dynamics derived from the equivariance of the system, and applying extended Kalman filter design principles~\cite{vangoorEquivariantFilterEqF2023}.
For a system $\dot{\xi} = f_{u}(\xi)$ that is equivariant in the sense of~\eqref{eq:EquivarianceDefinition}, the dynamics are carried onto the group by an \emph{equivariant lift} $\Lambda : \M \times \L \rightarrow \g$ satisfying the lift condition~\cite{vangoorEquivariantFilterEqF2023},
\begin{align}
\label{eq:LiftCondition}
    \mathrm{D}\phi_{\xi}\big[\Lambda(\xi, u)\big] = f_{u}(\xi), \quad \forall\, \xi \in \M,\ u \in \L.
\end{align}
The observer state $\hat{X} \in \G$ reproduces the state trajectory through the action $\phi$, and the EqF carries its estimate about a chosen \emph{origin} $\xi^{o} \in \M$ as $\hat{\xi} = \phi(\hat{X}, \xi^{o}) \in \M$. 
The observer state $\hat{X}$ evolves on the lifted system as the internal model for the observer dynamics~\cite{vangoorEquivariantFilterEqF2023},
\begin{align}
\label{eq:ObserverDynamics}
    \dot{\hat{X}} = \mathrm{D}L_{\hat{X}} \big[\Lambda(\hat{\xi}, u)\big] + \mathrm{D}R_{\hat{X}}[\Delta],
\end{align}
where $\mathrm{D}L_{\hat{X}}$ and $\mathrm{D}R_{\hat{X}}$ are the differentials at the identity of the left and right translations by $\hat{X}$. The lift $\Lambda$ propagates the estimate, while the correction $\Delta \in \g$ is driven by an innovation in the output of the system. 

The EqF tracks the \emph{global equivariant error},
\begin{align}
\label{eq:GlobalError}
    e \coloneq \phi(\hat{X}^{-1}, \xi) \in \M,
\end{align}
which equals the origin $\xi^o$ exactly when $\hat{\xi} = \xi$. 
The EqF is thus designed to drive the error to the origin~\cite{vangoorEquivariantFilterEqF2023}.

Sensor measurements are described by an output function $h: \M \rightarrow \N$ onto an output manifold $\N$.
The output $h$ is called \emph{equivariant} with respect to the action $\phi$ if there exists a group action $\rho : \G \times \N \rightarrow \N$ satisfying
\begin{align}
\label{eq:OutputEquivariance}
    \rho\big(X, h(\xi)\big) = h\big(\phi(X, \xi)\big), \quad \forall\, X \in \G,\ \xi \in \M.
\end{align}
When the output is equivariant, the linearization of the output function can be improved to obtain a third-order approximation~\cite{vangoorEquivariantFilterEqF2023}.

\section{System Model}
\label{sec:SystemModel}

The system model for the AUV is an INS with IMU biases augmented with a virtual velocity channel that renders it equivariant under the Tangent-Group symmetry as described in~\cite{fornasierEquivariantFilterDesign2022a}. Let $\{G\}$ be a reference frame with its third axis aligned with gravity, and $\{B\}$ the body-fixed IMU frame. 
The navigation states $R$, $v$, and $p$ are the attitude, velocity, and position of $\{B\}$ relative to $\{G\}$. 
The IMU inputs, angular velocity $\omega$ and specific force $a$, the virtual velocity input $\nu$ and the corresponding biases $b_{\omega}, b_{a}, b_{\nu}$ are expressed in $\{B\}$. 
Gravity $g$ is expressed in $\{G\}$.
The augmented states and inputs lie on the manifolds $\M$ and $\L$, respectively, with elements,\begin{subequations}
\begin{align}
\label{eq:StateSpaceDefinition}
    \M         &\coloneq \SEk{2}{3} \times \R{9},    & \xi        &\coloneq (T, b) \in \M, \\
    T          &\coloneq (R, v, p) \in \SEk{2}{3},   & b          &\coloneq (b_\omega, b_a, b_\nu) \in \R{9}, \\
\label{eq:InputSpaceDefinition}
    \L         &\coloneq \R{9} \times \R{9},         & u          &\coloneq (\mathrm{w}, \tau) \in \L, \\
    \mathrm{w} &\coloneq (\omega, a, \nu) \in \R{9}, & \tau &\coloneq (\tau_{\omega}, \tau_{a}, \tau_{\nu}) \in \R{9}.
\end{align}
\end{subequations}

The INS dynamics are then given by~\cite{fornasierEquivariantFilterDesign2022a},
\begin{subequations}
\label{eq:INSDynamics}
\begin{align}
    \dot{R} &= R\cro{(\omega - b_{\omega})}, & \dot{b}_{\omega} &= \tau_{\omega}, \\
    \dot{v} &= R(a - b_{a}) + g,             & \dot{b}_{a}      &= \tau_{a}, \\
    \dot{p} &= R(\nu - b_{\nu}) + v,         & \dot{b}_{\nu}    &= \tau_{\nu},
\end{align}
\end{subequations}
where $\dot{b} = \tau$ encodes the bias random walk. 
The virtual inputs $\nu$ and ${\tau}_{\nu}$ are (practically) chosen to be zero, and the virtual bias state $b_{\nu}$ is initialized at zero, so that the original position dynamics $\dot{p} = v$ can be recovered exactly.
The dynamics~\eqref{eq:INSDynamics} are recast in matrix form~\cite{fornasierEquivariantSymmetriesInertial2025},
\begin{subequations}
\label{eq:SystemDynamics}
\begin{align}
    \dot{T} &= T(W - B + N) + (G - N)T \\
    \dot{b} &= \tau,
\end{align}
\end{subequations}
where the navigation states $T = (R, v, p)$ are embedded in $\SEk{2}{3}$ as a $5\times5$ matrix~\eqref{eq:SE23}.
The input $W$, bias $B$, and gravity $G$ matrices are obtained via the wedge map~\eqref{eq:SE23}, $W \coloneq \wed{\mathrm{w}}$, $B \coloneq \wed{b}$, and $G \coloneq \wed{\mathbf{g}}$ with $\mathbf{g} \coloneq (\zero{3}{1},\ g,\ \zero{3}{1}) \in \R{9}$.
The structural matrix $N \in \R{5\times5}$ is zero except for a one in its fourth row and fifth column, encoding the drift component $\dot{p} = v$ of the INS dynamics~\eqref{eq:INSDynamics}.

\section{Tangent-Group Symmetry}
\label{sec:SymmetryGroup}

The Tangent-Group (TG) symmetry~\cite{fornasierEquivariantFilterDesign2022a} embeds the intrinsic structure of the navigation states and IMU biases inside the group.
It is the semi-direct product~(\cref{sec:LieGroups}) of the extended-pose group $\SEk{2}{3}$ and its Lie algebra $\se{2}{3}$,
\begin{align}
\label{eq:GroupDefinition}
    \G \coloneq \mathbf{TG} \coloneq \SEk{2}{3} \ltimes \se{2}{3}.
\end{align}
An element $X \coloneq (A, c) \in \G$ of the group $\G \coloneq \G_{\mathbf{TG}}$ is composed of a navigation part $A \in \SEk{2}{3}$ and a bias part $c \in \se{2}{3}$. 
The navigation part $A$ is further broken down into the attitude, velocity, and position components $(R_{A}, v_{A}, p_{A})$.
The product, inverse, exp, log, and adjoint maps for the TG symmetry are those of~\cref{sec:LieGroups} with $\Hg \coloneq \SEk{2}{3}$.

The system dynamics~\eqref{eq:SystemDynamics} satisfy the equivariance condition~\eqref{eq:EquivarianceDefinition} under the TG symmetry with a state action $\phi \coloneq \phi_{\mathbf{TG}}$ and a compatible input action $\psi \coloneq \psi_{\mathbf{TG}}$.
For $\xi \in \M$ and $u \in \L$,
\begin{subequations}
\label{eq:GroupActions}
\begin{align}
\label{eq:StateAction}
    \phi(X, \xi) &\coloneq \big(TA,\ \ve{\Ad}{A^{-1}}(b - \v{c})\big), \\
\label{eq:InputGroupAction}
    \psi(X, u) &\coloneq \big(\ve{\Ad}{A^{-1}}(\mathrm{w} - \v{c}) + \Omega,\ \ve{\Ad}{A^{-1}}\tau\big),
\end{align}
\end{subequations}
where $\Omega \coloneq (\zero{3}{1},\ \zero{3}{1},\ -R_{A}^{\T}v_{A})$ encodes the position drift component $\dot{p} = v$ of the INS dynamics~\eqref{eq:INSDynamics}.

The equivariant lift $\Lambda: \M \times \L \rightarrow \g,\ \Lambda = (\Lambda_{1}, \Lambda_{2})$ for the TG symmetry is given by~\cite{fornasierEquivariantFilterDesign2022a},
\begin{subequations}
\label{eq:EquivariantLift}
\begin{align}
    \Lambda_1(\xi, u) &\coloneq (W - B + N) + T^{-1}(G - N)T, \\
    \Lambda_2(\xi, u) &\coloneq \ad_{B}\big[\Lambda_1(\xi, u)\big] - \wed{\tau},
\end{align}
\end{subequations}
which lifts the INS dynamics~\eqref{eq:INSDynamics} to the observer dynamics~\eqref{eq:ObserverDynamics} on TG.

\begin{table}[t]
    \centering
    \small
    \caption{Linearization error orders for dynamics~\cite[Table~2]{fornasierEquivariantSymmetriesInertial2025} and measurement updates (\cref{sec:MeasurementModels,sec:Linearization}). The TG-EqF uniquely achieves exact navigation state linearization and a third-order DVL update. ($^{\dagger}$The MEKF carries biases and position in Euclidean coordinates, making $\dot{\varepsilon}_{b} = 0$ and the depth update trivially exact.)}
    \label{tab:LinearizationOrder}
    \begin{tabular}{@{}lcccccc@{}}
    \toprule
    & \multicolumn{4}{c}{Error dynamics} & \multicolumn{2}{c}{Update} \\
    \cmidrule(lr){2-5}\cmidrule(lr){6-7}
    Filter & $\varepsilon_{R}$ & $\varepsilon_{v}$ & $\varepsilon_{p}$ & $\varepsilon_{b}$ & DVL & Depth \\
    \midrule
    MEKF                   & 2nd   & 2nd   & exact & exact$^{\dagger}$ & 2nd & exact$^{\dagger}$ \\
    TFG-IEKF               & exact & 2nd   & 2nd   & 2nd & 2nd & 2nd \\
    \textbf{TG-EqF (ours)} & exact & exact & exact & 2nd & 3rd & 2nd \\
    \bottomrule
\end{tabular}

\end{table}

\section{Measurement Models}
\label{sec:MeasurementModels}

Output models for the acoustic DVL sensor $h_d$ and the pressure-derived depth sensor $h_p$ are presented in this section. 
The DVL output admits an equivariant output action while the depth output does not. 
The output matrices that linearize them are derived in~\cref{sec:Linearization}.

\subsection{Doppler Velocity Log (DVL)}
\label{sec:DVLMeasurementModel}

The DVL admits a third-order accurate approximation because its 3\,DoF body-frame velocity output admits an output action that is equivariant in the sense of~\eqref{eq:OutputEquivariance}.
This is unique among the filters compared in~\cref{tab:LinearizationOrder} and a novel contribution of the present work.
A DVL sensor emits acoustic beams in a fixed configuration down to the sea floor and measures Doppler shift in the reflected echoes to calculate its velocity relative to the bottom.
Formally, the DVL measurement is given by
\begin{align}
\label{eq:DVLOutput}
    y_{d} = h_{d}(\xi) = R^{\T}v \in \N_{d},
\end{align}
where $h_{d} : \M \rightarrow \N_{d}$ is the measurement function with output space $\N_{d} \coloneq \R{3}$.

\begin{lemma}[DVL output action]
\label{lem:DVLOutputAction}
The map $\rho_{d} : \G \times \N_{d} \rightarrow \N_{d}$,
\begin{align}
\label{eq:DVLOutputEquivariance}
    \rho_{d}(X, y_{d}) \coloneq R_{A}^{\T}(y_{d} + v_{A}),
\end{align}
is a right action of $\G$ on $\N_{d}$, and $h_{d}$ is equivariant with respect
to $\phi$ and $\rho_{d}$ in the sense of~\eqref{eq:OutputEquivariance}.
\end{lemma}

\begin{proof}
It is straightforward to show that $\rho_{d}$ satisfies the right action axioms~\eqref{eq:RightActionAxioms}.
Equivariance follows from the state action $\phi$~\eqref{eq:StateAction} acting by $T \mapsto TA$:
\begin{align*}
    h_{d}(\phi(X, \xi)) &= (RR_{A})^{\T}(v + Rv_{A}) \\
                        &= R_{A}^{\T}(R^{\T}v + v_{A}) \\
                        &= \rho_{d}(X, h_{d}(\xi)),
\end{align*}
for all $X \in \G$ and $\xi \in \M$.
\end{proof}

This equivariance has two key consequences.
First, the measurement may be transported to the reference output ${y}_{d}^{o} \coloneq h_{d}(\xi^{o})$ via
$\tilde{y}_{d} \coloneq \rho_{d}(\hat{X}^{-1}, y_{d})$, so the innovation
\begin{align}
\label{eq:DVLInnovation}
    \iota_{d} \coloneq \tilde{y}_{d} - {y}_{d}^{o}
\end{align}
is formed at the state-independent point ${y}_{d}^{o} \in \N_{d}$. 
Second, the output can be approximated to third order (rather than second order as in standard filters) using \cite[Lemma~5.3]{vangoorEquivariantFilterEqF2023}, establishing the claim of~\cref{tab:LinearizationOrder}.

\subsection{Pressure Sensor}
\label{sec:DepthMeasurementModel}

The depth output admits no equivariant action, and its update is correspondingly second order~(\cref{tab:LinearizationOrder}). 
The pressure sensor measures depth, the component of position along the gravity-aligned vertical axis of $\{G\}$, through the output function $h_{p} : \M \rightarrow \N_{p}$, with output space $\N_{p} \coloneq \mathbb{R}$,
\begin{align}
\label{eq:DepthOutput}
    y_{p} = h_{p}(\xi) = e_{3}^{\T}p \in \N_{p}.
\end{align}
The obstruction to constructing an equivariant output action is that the state action~\eqref{eq:StateAction} sends $e_{3}^{\T}p \mapsto e_{3}^{\T}(p + Rp_{A})$, and the resulting extra term $e_{3}^{\T}Rp_{A}$ depends on the state $R$ and thus cannot be written in terms of the original output $e_{3}^{\T}p$.
In the absence of equivariance, the output innovation is simply taken as the scalar residual at the current estimate,
\begin{align}
\label{eq:DepthInnovation}
    \iota_{p} \coloneq y_{p} - h_{p}(\hat{\xi}),
\end{align}
and approximated to second order.

\section{Equivariant Filter Design}
\label{sec:FilterDesign}

The EqF of~\cref{sec:EqFBackground} is instantiated for the acoustic- and pressure-aided INS of~\cref{sec:SystemModel} with the TG symmetry of~\cref{sec:SymmetryGroup} and the outputs of~\cref{sec:MeasurementModels}.
The resulting observer state $\hat{X} \in \G$ evolves under the dynamics~\eqref{eq:ObserverDynamics} with lift~\eqref{eq:EquivariantLift} and correction terms derived from the DVL output~\eqref{eq:DVLOutput} and depth output~\eqref{eq:DepthOutput}.
The reference origin $\xi^{o} \in \M$ is fixed at the identity element $(I_{\SEk{2}{3}}, \zero{9}{1})$ of the state manifold throughout.
The filter estimate is given by the state acting on the observer state $\hat{\xi} = \phi(\hat{X}, \xi^{o})$. 
This section presents the EqF process and output matrices.

\subsection{Equivariant Error and Its Dynamics}
\label{sec:ErrorDynamics}

The global equivariant error~\eqref{eq:GlobalError} splits into an extended-pose block and a bias block,
\begin{align}
\label{eq:ErrorSplit}
    e &\coloneq (e_T, e_b) = \big(T\hat{T}^{-1},\ \ve{\Ad}{\hat{T}}(b - \hat{b})\big).
\end{align}
The error dynamics are derived by differentiating~\eqref{eq:ErrorSplit} along the INS dynamics~\eqref{eq:SystemDynamics} and given by~\cite{fornasierEquivariantSymmetriesInertial2025},
\begin{subequations}
\label{eq:ErrorDynamics}
\begin{align}
    \dot{e}_T &= (G - N)e_T - e_T(G - N) + T\wed{(\hat{b} - b)}\hat{T}^{-1}, \\
    \dot{e}_b &= \ve{\Ad}{\hat{T}}\,\ve{\ad}{\v{\Lambda_1(\hat{\xi}, u)}}\,(b - \hat{b}).
\end{align}
\end{subequations}
Note that the bias error dynamics are driven only by the bias error, in contrast to the alternative filters~\cite[Table~2]{fornasierEquivariantSymmetriesInertial2025} where the bias error dynamics are coupled to the navigation error.

\subsection{Error Coordinates and Linearization}
\label{sec:Linearization}

The error is linearized in the canonical logarithmic coordinates of the Tangent Group with the chart $\vartheta : \mathcal{U}_{\xi^{o}} \subset \M \rightarrow \R{18}$~\cite[Eq.~(32)]{fornasierEquivariantSymmetriesInertial2025},
\begin{align}
\label{eq:ErrorChart}
    \varepsilon \coloneq \vartheta(e) \coloneq \v{\log\big(\phi_{\xi^{o}}^{-1}(e)\big)},
\end{align}
where the projection $\phi_{\xi^{o}}(X) \coloneq \phi(X, \xi^{o})$ of~\cref{sec:GroupActions} is invertible since $\phi$ is free and transitive.

Linearizing~\eqref{eq:ErrorDynamics} about $\varepsilon = 0$ yields the first-order model $\dot{\varepsilon} \approx \mathbf{A}_{t}^{o}\varepsilon$, with the state matrix $\mathbf{A}_{t}^{o}$ derived from the chain of differentials given in~\cite[Lemma~V.2]{vangoorEquivariantFilterEqF2023} and evaluated at the origin input $u^{o} = \psi(\hat{X}^{-1}, u)$.
The resulting state matrix~\cite[Sec.~VI.A]{fornasierEquivariantFilterDesign2022a} is block-upper-triangular,
\begin{align}
\label{eq:StateMatrix}
    \mathbf{A}_{t}^{o} \coloneq \begin{bmatrix}
        \bm{\Upsilon} & I_{9} \\
        \mathbf{0}    & \ve{\ad}{(\mathrm{w}^{o} + \mathbf{g})}
    \end{bmatrix}, \quad
    \bm{\Upsilon} \coloneq \begin{bmatrix}
        \mathbf{0} & \mathbf{0} & \mathbf{0} \\
        \cro{g}    & \mathbf{0} & \mathbf{0} \\
        \mathbf{0} & I_{3}      & \mathbf{0}
    \end{bmatrix},
\end{align}
with $\mathrm{w}^{o}$ the origin IMU-input bundle. 
This is the matrix form of the TG-EqF error dynamics tabulated in~\cite[Table~2]{fornasierEquivariantSymmetriesInertial2025}. 
The constant coupling makes the navigation-state error exactly linear, leaving a second-order linearization error in the bias block alone~\cite{fornasierEquivariantSymmetriesInertial2025}.
Note that the chart $\vartheta$ reverses the sign of the bias block, $\varepsilon_{b} = -\v{e_{b}} + O(\|\varepsilon\|^{2})$, which accounts for the positive coupling, a $+I_9$ in the upper right block of~\eqref{eq:StateMatrix}.

The output matrices linearize the DVL and depth outputs of~\cref{sec:MeasurementModels} about $\varepsilon = 0$ and are derived from the chain of differentials in~\cite[Eq.~33]{vangoorEquivariantFilterEqF2023}.
The DVL output~\eqref{eq:DVLOutput}, with the compatible output action $\rho_{d}$, yields an equivariant output matrix~\cite[Lemma~V.3]{vangoorEquivariantFilterEqF2023}, which is third-order approximate,
\begin{align}
\label{eq:DVLOutputMatrix}
    \mathbf{C}_{d}^{\star} &\coloneq \begin{bmatrix}
        \tfrac{1}{2}\cro{(y_{d}^{o} + \tilde{y}_{d})} & I_{3} & \zero{3}{3} & \zero{3}{9}
    \end{bmatrix} \in \R{3\times18},
\end{align}
where $y_{d}^{o} = h_d(\xi^o) = 0$ at the chosen reference origin $\xi^o$.

The depth output~\eqref{eq:DepthOutput} yields a standard output matrix, which is second-order approximate,
\begin{align}
\label{eq:DepthOutputMatrix}
    \mathbf{C}_{p}^{o} &= \begin{bmatrix}
        (\hat{p} \times e_{3})^{\T} & \zero{1}{3} & e_{3}^{\T} & \zero{1}{9}
    \end{bmatrix} \in \R{1\times18}.
\end{align}

\subsection{Filter Equations}
\label{sec:FilterEquations}

The state matrix $\mathbf{A}_{t}^{o}$ is time varying, depending on the estimate only through the transported input $\mathrm{w}^{o}$. 
The input is assumed to be constant over a sampling interval $\Delta t$, which gives $\mathbf{A}_{t}^{o}$ on the interval $[t_{k-1}, t_{k}]$. 
The covariance flow is given by the ODE $\dot{\mathbf{\Sigma}} = \mathbf{A}_{t}^{o}\mathbf{\Sigma} + \mathbf{\Sigma}\,\mathbf{A}_{t}^{o\T} + \mathbf{P}$, whose solution is propagated by the transition matrix $\bm{\Phi} \coloneq \exp(\mathbf{A}_{t}^{o}\Delta t)$, evaluated numerically up to the fourth-order term.
The process-noise gain is given by $\mathbf{Q} = \mathbf{B}_{t}\,\mathbf{P}^{m}\,\mathbf{B}_{t}^{\T}$, with $\mathbf{Q}^{m}$ composed of the IMU white noise and random-walk noise densities, and is integrated to first order by the process noise matrix $\mathbf{Q}_{\mathrm{d}} \approx \mathbf{P}\,\Delta t$.
The input matrix $\mathbf{B}_{t}$ is obtained by linearizing the error dynamics with respect to a perturbation of the input~\cite[Eq.~42]{vangoorEquivariantFilterEqF2023},
\begin{align}
\label{eq:ProcessNoiseGainBlocks}
    \mathbf{B}_{t} &= \begin{bmatrix}
        \ve{\Ad}{\hat{T}}\big|_{:,1:6} & \zero{9}{6} \\
        \zero{9}{6} & -\ve{\Ad}{\hat{T}}\big|_{:,1:6}
    \end{bmatrix} \in \R{18 \times 12}.
\end{align}
The slice $|_{:,1:6}$ keeps only the $\omega, a$ columns of $\ve{\Ad}{\hat{T}}$.
The virtual velocity input $\nu$ is set to zero throughout.
The corresponding columns are dropped and the covariance is driven by the six physical IMU noise densities alone.

The filter is a continuous--discrete EqF, and the propagation and correction steps take the usual Kalman-filter form~\cite{fornasierEquivariantSymmetriesInertial2025}.
Over the interval $[t_{k-1}, t_k]$, the observer state propagates by a step of~\eqref{eq:ObserverDynamics}, $\hat{X}_{k} = \hat{X}_{k-1}\,\exp\!\big(\Lambda(\hat{\xi}, u)\,\Delta t\big)$, and the covariance by the transition matrix $\bm{\Phi}$ and process noise matrix $\mathbf{Q}_{\mathrm{d}}$.

At a measurement, the covariance $\mathbf{\Sigma}_{k}^{+}$ and Kalman gain $\mathbf{K}$ follow the discrete EKF equations in the innovation $\iota$, output matrix $\mathbf{C}$, and noise covariance $\mathbf{R}$, and the correction $\Delta = \wed{\big(\mathbf{K}\,\iota\big)}$ is applied on the group as $\hat{X}_{k}^{+} = \exp(\Delta)\,\hat{X}_{k}$ through the exponential~\eqref{eq:ExpMap}.
The DVL update uses the equivariant output matrix $\mathbf{C}_{d}^{\star}$~\eqref{eq:DVLOutputMatrix} and the measurement noise covariance matrix $\mathbf{R}_{d}$ derived by transporting the DVL noise covariance to the reference output origin $y_{d}^{o}$ via the output action $\rho_d$.
The depth update uses the standard output matrix $\mathbf{C}_{p}^{o}$~\eqref{eq:DepthOutputMatrix} and the measurement noise covariance matrix $\mathbf{R}_{p}$.

The update displaces the estimate on $\M$, so the covariance must be transported into the chart about the corrected estimate. The EqF reset step~\cite{geGeometryExtendedKalman2025} does this with the left Jacobian of~\cref{sec:LieGroups}, $\mathbf{J} = J_{l}(\v{\Delta})$, and the covariance update $\mathbf{\Sigma}_{k}^{+} \leftarrow \mathbf{J}\,\mathbf{\Sigma}_{k}^{+}\,\mathbf{J}^{\T}$, which smoothly approaches identity as $\Delta \to 0$.

\section{Simulations and Experimental Validation}
\label{sec:Experiments}

The TG-EqF is evaluated via Monte Carlo simulations against two filters: the TFG-IEKF with the Two-Frame-Group symmetry~\cite{barrauGeometryNavigationProblems2023, fornasierEquivariantSymmetriesInertial2025} and the MEKF with the Special Orthogonal Group symmetry~\cite{fornasierEquivariantSymmetriesInertial2025}.
All three filters receive identical, seed-matched measurements, initial estimates, initial covariances, IMU noise densities, and measurement noise covariances.
\Cref{sec:SimulationSetup} describes the simulation setup, and \cref{sec:ConsistencyResults,sec:AccuracyResults} present the consistency and accuracy results.
\Cref{sec:FieldEvaluation} evaluates the filters on AUV field data.

\subsection{Simulation Setup}
\label{sec:SimulationSetup}

We reconstruct a $300$~s trajectory from a real AUV survey dive (\cref{fig:SeekerAUV}), with a $120$~m path at roughly $26$~m depth that varies by approximately $5$~m (\cref{fig:DriftGrowth}).
Cubic $C^{2}$ splines are fitted to the offline factor-graph smoother solution, which also fuses position measurements from an Ultra-Short Baseline (USBL) sensor and heading measurements from a magnetometer.
These splines are differentiated to give the true trajectory, IMU inputs, and aiding measurements for simulations.
The vehicle averages $0.40$~m/s along a lawnmower pattern of three straight segments and three $\approx 170^{\circ}$ turns, at yaw rates up to $0.29$~rad/s.
These turns excite the attitude-velocity coupling, the main term that the TG and TFG symmetries treat differently (\cref{tab:LinearizationOrder}).

IMU data are generated at $100$~Hz using noise densities measured directly from the vehicle's instrument, $\sigma_{\omega} = (4.5, 3.2, 1.7) \times 10^{-4}$~rad/s/$\sqrt{\text{Hz}}$ and $\sigma_{a} = (4.5, 6.3, 8.9) \times 10^{-3}$~m/s$^{2}$/$\sqrt{\text{Hz}}$.
The true biases are zero and constant, with no bias random walk.
Aiding measurements carry zero-mean Gaussian noise: DVL velocity at $3$~Hz with isotropic standard deviation $\sigma_{\mathrm{DVL}} = 0.1$~m/s, and depth at $10$~Hz with $\sigma_{\mathrm{depth}} = 0.1$~m.
Note that horizontal position $p_x, p_y$ and heading are \emph{unobservable} with this sensor suite~\cite{kleinObservabilityAnalysisDVL2015}.

Consistency is evaluated with the Average Normalized Estimation Error Squared (ANEES), the average value of the filter's error ${\varepsilon^{\T}}{\mathbf{\Sigma}}^{-1}\varepsilon$ normalized over the number of seeds $N=100$ and the dimension of the error block $\dim(\varepsilon)$~\cite[Section~5.4.2]{bar-shalomEstimationApplicationsTracking2001}.
It equals one when a filter's reported covariance matches its actual error distribution, and exceeds one when the filter is overconfident.
Over $100$ seeds, the $95\%$ $\chi^{2}$ interval of the ANEES is $[0.85, 1.17]$ for a three-dimensional block and $[0.91, 1.09]$ for the nine-dimensional navigation block $(R, v, p)$.
The initial estimate is offset from the truth by a zero-mean Gaussian draw with standard deviations $\sigma_{R} = 5^{\circ}$, $\sigma_{v} = 10^{-2}$~m/s, $\sigma_{p} = 10^{-2}$~m, $\sigma_{b_\omega} = 10^{-2}$~rad/s, and $\sigma_{b_a} = 10^{-1}$~m/s$^{2}$.
The same draw is applied to all three filters, as a world-frame rotation for attitude and an additive offset for the other states, and each filter builds its initial covariance from the same standard deviations.
The TG-EqF's virtual bias $b_\nu$, which has no physical counterpart, is initialized with $\sigma_{b_\nu} = 10^{-6}$~m/s.

Accuracy is evaluated with the RMSE over time of the physical errors, averaged over the $N$ seeds, using the geodesic angle for attitude and the Euclidean norm for the other states.
Each estimated trajectory is first aligned to the reference by an $\SE{2}$ transform at $t = 0$, which removes the unobservable offsets in initial horizontal position and heading.
Additionally, we report how many seeds each filter wins by attaining the lowest error of the three, and each filter's coverage: the number of seeds whose ANEES falls inside the $95\%$ $\chi^{2}$ interval of that filter's own covariance, nominally $95$ seeds of $100$.

\begin{figure}[t]
    \centering
    \includegraphics[width=\linewidth]{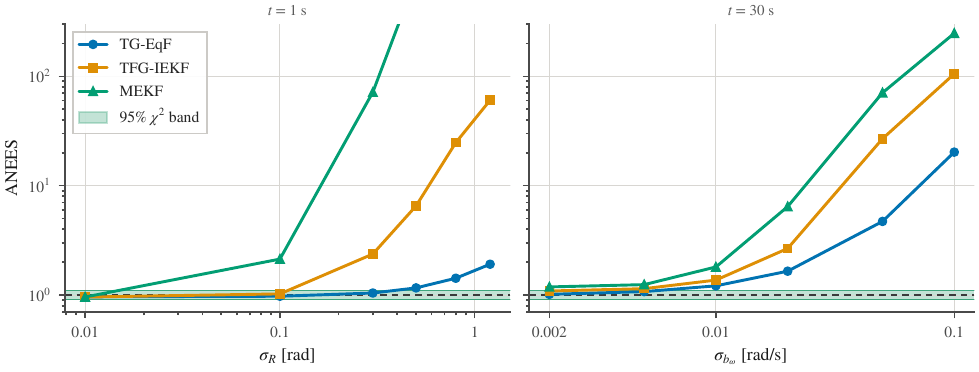}
    \caption{ANEES of the navigation block $(R, v, p)$ as the initial attitude uncertainty (left, read at $t = 1$~s) and the initial gyroscope-bias uncertainty (right, read at $t = 30$~s) are swept, with $N = 100$ seeds per point. Each panel is read before the unobservable heading degrades the filters. The TG-EqF stays closest to one in both sweeps, while the baselines diverge as their second-order linearization errors grow (\cref{tab:LinearizationOrder}).}
    \label{fig:InitSweep}
\end{figure}

\begin{figure}[t]
    \centering
    \includegraphics[width=\linewidth]{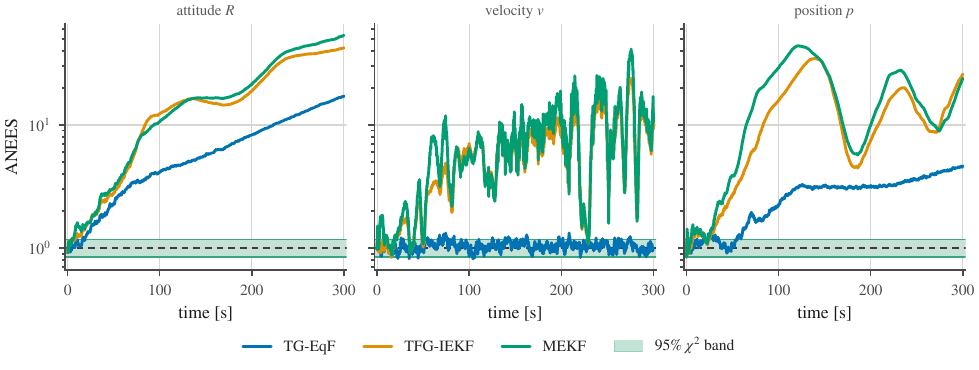}
    \caption{ANEES of the attitude, velocity, and position blocks over the $300$~s run, $N = 100$ seeds. A consistent filter sits at one, inside the shaded $95\%$ $\chi^{2}$ interval. The TG-EqF's velocity ANEES stays near one for the whole run, while that of the baselines climbs to roughly ten. In attitude and position all three filters leave the interval: at $t = 300$~s the TG-EqF reaches $17$ and $4.7$, and the baselines more than $40$ and about $25$.}
    \label{fig:Consistency}
\end{figure}

\begin{figure}[t]
    \centering
    \includegraphics[width=\linewidth]{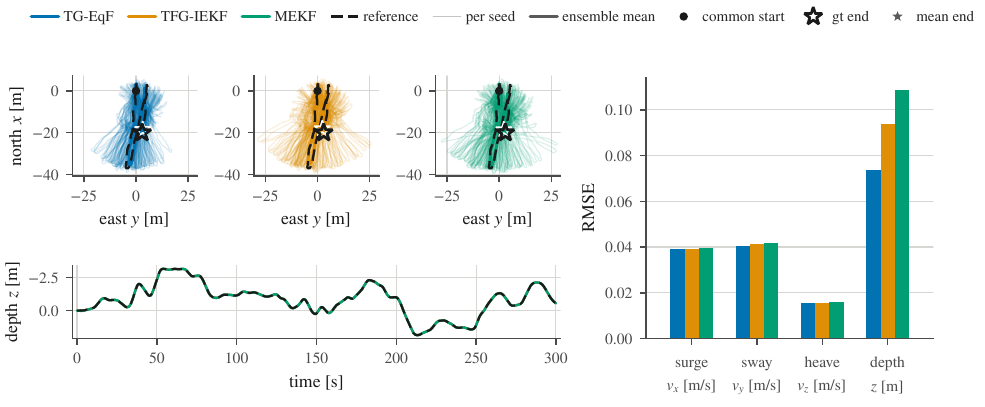}
    \caption{Spatial drift and RMSE over the $300$~s run ($N = 100$). Top left: estimated trajectories of each seed (light) against the reference (dashed). The TG-EqF's spread is visibly tighter, matching its lower position RMSE in \cref{tab:Accuracy}. Bottom left: estimated depth against the reference. Right: RMSE of the directly measured states. The velocity RMSE is nearly identical across filters, while the MEKF's depth RMSE is $48\%$ higher than the TG-EqF's.}
    \label{fig:DriftGrowth}
\end{figure}

\begin{table}[tb]
    \centering
    \small
    \caption{Navigation accuracy over the $300$~s run ($N = 100$). Left: RMSE (bold denotes the minimum if it improves upon the next best by $\ge 2\%$). Right: the number of seeds on which each filter attains the lowest error of the three.}
    \label{tab:Accuracy}
    \begin{tabular}{lrrrrrr}
    \toprule
     & \multicolumn{3}{c}{RMSE} & \multicolumn{3}{c}{Seeds won} \\
    \cmidrule(lr){2-4} \cmidrule(lr){5-7}
     & TG-EqF & TFG-IEKF & MEKF & TG-EqF & TFG-IEKF & MEKF \\
    \midrule
    $R$ [rad] & \textbf{0.321} & 0.416 & 0.429 & \textbf{56} & 21 & 23 \\
    $v$ [m/s] & \textbf{0.142} & 0.174 & 0.177 & \textbf{56} & 20 & 24 \\
    $p$ [m] & \textbf{4.02} & 5.18 & 5.24 & \textbf{51} & 23 & 26 \\
    $b_\omega$ [mrad/s] & \textbf{4.04} & 4.59 & 4.74 & \textbf{58} & 21 & 21 \\
    $b_a$ [mm/s$^2$] & 28.7 & 28.8 & 30.5 & \textbf{46} & 29 & 25 \\
    \bottomrule
\end{tabular}

\end{table}

\subsection{Consistency}
\label{sec:ConsistencyResults}

The three filters differ most in the covariance they report, and the difference follows from the order of the linearization error (\cref{tab:LinearizationOrder}).
\Cref{fig:InitSweep} reports the ANEES of the navigation block against the initial attitude and gyroscope-bias uncertainties, one at a time, while holding the other state uncertainties constant.
The left panel sweeps the initial attitude uncertainty over two decades, from $\sigma_{R} = 10^{-2}$ to $1.2$~rad, and reads the ANEES at $t = 1$~s.
At $\sigma_{R} = 10^{-2}$~rad all three filters are consistent, since the higher-order terms in the error dynamics are negligible.
As $\sigma_{R}$ grows, the MEKF leaves the $95\%$ interval first and reaches an ANEES of $6.1 \times 10^{4}$ at $\sigma_{R} = 1.2$~rad, while the TFG-IEKF reaches $60$.
The TG-EqF stays inside the interval up to $\sigma_{R} = 0.3$~rad and reaches only $1.9$ at $\sigma_{R} = 1.2$~rad.
The right panel sweeps the initial gyroscope-bias uncertainty from $2 \times 10^{-3}$ to $10^{-1}$~rad/s and reads the ANEES at $t = 30$~s, once the bias error has acted through the first turn.
This range includes the gyroscope bias measured on the IMU of the real AUV, $6.1 \times 10^{-3}$~rad/s, near which the ANEES of all three filters remains close to one, although only the TG-EqF lies inside the $95\%$ interval.
At the maximum of $10^{-1}$~rad/s no filter is consistent, but the TG-EqF's ANEES of $20$ is still five times below the TFG-IEKF's $105$.
The TG-EqF's consistency therefore degrades far less than the baselines' as either initial error grows.
Both results follow from its exact navigation error dynamics (\cref{tab:LinearizationOrder}), whereas both baselines linearize the velocity error dynamics only to second order.

\Cref{fig:Consistency} plots the ANEES against time for the uncertainties described in \cref{sec:SimulationSetup}.
The TG-EqF's velocity ANEES stays within $[0.82, 1.33]$ and lies inside the $95\%$ $\chi^{2}$ interval for $91\%$ of the run, against $9\%$ for the TFG-IEKF and $6\%$ for the MEKF.
The performance in terms of coverage has the same ordering: at $t = 300$~s the TG-EqF's velocity covariance covers all $100$ seeds, against $39$ for the TFG-IEKF and $37$ for the MEKF.
The TG-EqF thus keeps its velocity covariance consistent over five minutes of dead reckoning.

Attitude and position consistency degrade in all three filters, because no measurement corrects the heading error and that error feeds into position error.
At $t = 300$~s the attitude ANEES is $17.1$ for the TG-EqF, against $42.4$ for the TFG-IEKF and $53.5$ for the MEKF, and the position ANEES is $4.7$, against $25.9$ and $23.9$.
On the navigation block $(R, v, p)$ the three filters end at $7.4$, $17.5$, and $36.0$, with coverage of $27$, $13$, and $17$ seeds against a nominal $95$.
The baselines' position ANEES also swings with the trajectory, peaking near the far end of a turn, at $t \approx 120$--$140$ and $233$~s.
The TG-EqF's position ANEES instead stays between $2.3$ and $3.7$ from $100$ to $250$~s.
The turns couple attitude error into velocity error, whose dynamics both baselines linearize only to second order (\cref{tab:LinearizationOrder}), and the velocity error then integrates into position; the TG-EqF's exact velocity error dynamics avoid this.

\subsection{Accuracy}
\label{sec:AccuracyResults}

The TG-EqF is the most accurate filter over the full run, with an RMSE $18$--$25\%$ lower than both baselines in the navigation states (\cref{tab:Accuracy}).
The seed-by-seed counts support this: the TG-EqF has the lowest navigation error of the three on $51$ to $56$ of the $100$ seeds depending on the state, against $20$ to $26$ for the other filters.
\Cref{fig:DriftGrowth} shows the same advantage: the TG-EqF's horizontal ensemble spread is visibly tighter, and its position RMSE is $22$--$23\%$ lower than the baselines'.
The body-frame velocity RMSE, which the DVL measures directly, is nearly identical across filters, yet the MEKF's depth RMSE is $48\%$ higher than the TG-EqF's.
Since the MEKF's depth update is exact (\cref{tab:LinearizationOrder}), this margin does not come from the update itself.
It comes from how each filter weights the depth measurement against its own estimate, which is set by the covariance and is in line with the TG-EqF's more consistent estimate of covariance.
In the biases, the TG-EqF's gyroscope-bias RMSE is $12\%$ lower than the TFG-IEKF's, with the separation opening after $5$~s (\cref{fig:BiasError}).

\begin{figure}[tb]
    \centering
    \includegraphics[width=\linewidth]{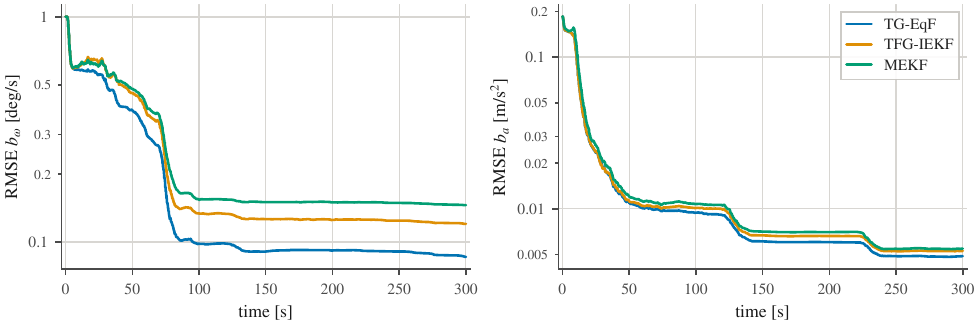}
    \caption{RMS bias estimation error over the $300$~s run, $N = 100$ seeds, with all three filters initialized identically. Left: gyroscope bias, the TG-EqF pulls away from the baselines from $t = 5$~s and keeps the gap to the end of the run; at $t = 300$~s its error is $1.4$ and $1.7$ times lower than theirs. Right: accelerometer bias, where the three filters stay within $12\%$ of one another at $t = 300$~s.}
    \label{fig:BiasError}
\end{figure}

This margin builds up over long horizons.
Over the first $30$~s the TG-EqF's position RMSE is only $6\%$ ($<3$~cm) lower than the TFG-IEKF's; the gap grows to $15\%$ over $60$~s and settles at $22\%$ for windows of two minutes or longer.
The exact linearization of the navigation error dynamics for the TG-EqF does not eliminate error growth, as unobservable heading error inevitably couples into position drift, but it significantly reduces the rate of that growth, making the TG-EqF relatively more reliable for long-horizon dead reckoning.
The relative performance is the same when a filter is poorly initialized in bias rather than in attitude (\cref{fig:InitSweep}).
At $\sigma_{b_\omega} = 5 \times 10^{-2}$~rad/s the TG-EqF's position RMSE is $0.90$~m, against $1.21$~m for the TFG-IEKF and $1.18$~m for the MEKF.

\subsection{Field Data Evaluation}
\label{sec:FieldEvaluation}

The practical viability of the TG-EqF is assessed on field data collected by an AUV surveying a complex reef environment (\cref{fig:SeekerAUV}).
Real-world datasets present unmodeled acoustic dropouts, sensor misalignment, and highly dynamic platform motion, offering a demanding test of the filter's theoretical properties.
The filters are evaluated over four distinct $300$~s survey windows (\cref{fig:FieldTrajectories}), where they are provided with $100$~Hz IMU data, sharing the instrument and noise densities of \cref{sec:SimulationSetup}, and aided by $3$~Hz DVL body velocity ($\sigma_{\mathrm{DVL}} = 0.033$~m/s) and $10$~Hz pressure depth ($\sigma_{\mathrm{depth}} = 0.1$~m). 
The same smoother (\cref{sec:SimulationSetup}) provides the reference.
To isolate the filter dynamics from initialization error, all three filters start from the reference pose and velocity, with the IMU biases set to their mission means estimated against this reference.

The relative performance of the different filters on the field data reflects the results from the Monte Carlo simulations.
The TG-EqF attains the lowest final position error in every window, outperforming the next-best filter by a median factor of $4.0$ with per-window ratios $1.6$--$4.3$.
The median final position errors of the TG-EqF, TFG-IEKF, and MEKF are $2.9$~m, $15.1$~m, and $9.9$~m respectively, while their median final attitude errors are $0.14$~rad, $2.16$~rad, and $1.16$~rad.
These results show that the TG-EqF exhibits reduced position drift over prolonged real-world dead reckoning compared to the alternative filters.

\begin{figure}[t]
    \centering
    \includegraphics[width=\linewidth]{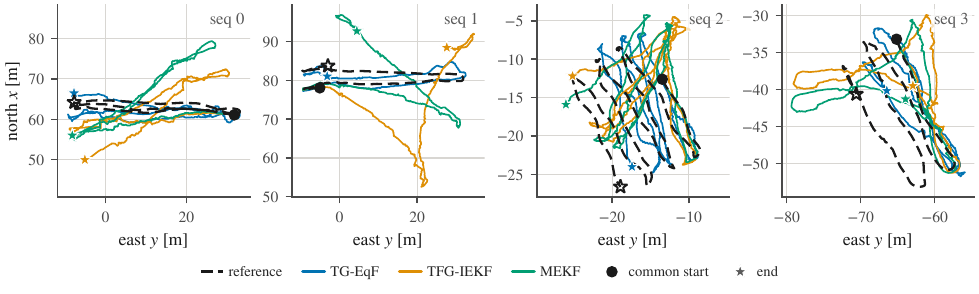}
    \caption{Field evaluation over four $300$~s survey windows.
    The filters use IMU data aided by DVL body velocity and pressure depth, but receive no absolute position information.
    Against the reference (dashed), the TG-EqF demonstrates greater stability and significantly reduced position drift over long horizons in real-world conditions.}
    \label{fig:FieldTrajectories}
\end{figure}

\section{Conclusion}
\label{sec:Conclusion}

This work presents the TG-EqF, an EqF for acoustic- and depth-aided inertial navigation that carries the IMU biases in its symmetry and uses an equivariant DVL output.
In Monte Carlo simulations against a TFG-IEKF and an MEKF, its consistency degrades least as initial attitude and gyroscope-bias errors grow, and it is the only filter whose velocity covariance stays consistent over the whole run.
Attitude and position consistency degrade in all three filters, least for the TG-EqF, which is also the most accurate in attitude, velocity, and position, with the advantage emerging within two minutes.
On field data from an AUV reef survey, the TG-EqF attains the lowest final position error in all four windows, by a median factor of $4.0$ over the next-best filter.
The formulation of the TG-EqF for AUV navigation and our implementation within GTSAM will facilitate extensions to incorporate USBL position and magnetometer heading to make the system fully observable for mission-time testing, explore error formulations~\cite{fornasierEquivariantSymmetriesInertial2025} for higher-order depth linearization, and run extensive field campaigns with repeated survey passes to assess real-world covariance consistency and accuracy, leading to deployment on onboard edge compute.

{\footnotesize
\bibliographystyle{IEEEtran}
\bibliography{references}
}

\end{document}